\documentclass[11pt,a4paper]{article}

\usepackage{fancyhdr}
\usepackage{amsmath,mathtools}
\usepackage{physics}
\usepackage{bm}
\usepackage{esint}
\usepackage{amsfonts}
\usepackage{amsthm}
\usepackage{amssymb}
\usepackage{amsbsy}
\usepackage{verbatim}
\usepackage{graphicx}
\usepackage{subfig}
\usepackage{url}
\usepackage{mathrsfs}
\usepackage[hmargin = 1in,vmargin=.75in]{geometry}
\usepackage{color}
\usepackage{amstext}
\usepackage{stmaryrd}
\usepackage{enumerate}
\usepackage{algpseudocode}
\usepackage{algorithm}
\usepackage{pifont}
\usepackage{prettyref}
\usepackage{hyperref}
\hypersetup{
    colorlinks = true,
}

\font\msbm=msbm10

\numberwithin{equation}{section}

\theoremstyle{plain}
\newtheorem{theorem}{Theorem}[section]
\newtheorem{lemma}[theorem]{Lemma}

\def\mathbb#1{\hbox{\msbm{#1}}}

\begin{document}
\title{\bf A theoretical guarantee for SyncRank}
\author{Yang Rao }
\maketitle

\begin{abstract}

We present a theoretical and empirical analysis of the SyncRank algorithm for recovering a global ranking from noisy pairwise comparisons. By adopting a complex-valued data model where the true ranking is encoded in the phases of a unit-modulus vector, we establish a sharp non-asymptotic recovery guarantee for the associated semidefinite programming (SDP) relaxation. Our main theorem characterizes a critical noise threshold—scaling as $\sigma = \mathcal{O}(\sqrt{n / \log n})$—below which SyncRank achieves exact ranking recovery with high probability. Extensive experiments under this model confirm the theoretical predictions and demonstrate the algorithm's robustness across varying problem sizes and noise regimes.
\end{abstract}



\section{Introduction}
The statistical ranking problem—inferring a global ordering of items from incomplete and noisy pairwise comparisons—emerges naturally in competitive sports analysis, preference aggregation, and economic exchange systems. Traditional approaches rooted in social choice theory often falter when confronted with modern datasets characterized by two pervasive challenges: (1) the comparisons are sparsely observed, with measurement graphs far from complete; (2) the noise exhibits strong heterogeneity, where certain subsets of comparisons may be significantly more reliable than others. These limitations are particularly evident in applications like soccer league standings analysis, where the outcome matrix contains both structured noise (e.g., home advantage biases) and random outliers.

Recent advances in group synchronization theory provide a novel geometric perspective for this classical problem. By mapping player ranks to phases on the unit circle and rank differences to angular offsets, the ranking task becomes equivalent to solving an instance of the angular synchronization problem over SO(2). This reformulation inherits key theoretical guarantees from the synchronization literature: spectral and semidefinite programming (SDP) relaxations can provably recover the underlying ranking under quantifiable noise thresholds. Crucially, the circular representation inherently handles cyclic inconsistencies through phase wrapping, circumventing the need for explicit outlier removal mechanisms required by linear embedding approaches.

Our synchronization-based framework unifies the treatment of cardinal and ordinal comparisons. For cardinal measurements (e.g., goal differences in soccer), we directly encode score offsets as angular displacements. For ordinal data (win/loss records), we construct similarity matrices through third-party comparison consistency scores before synchronization. The resulting algorithm demonstrates remarkable robustness across diverse noise models, including scenarios with adversarial outliers where a fraction of comparisons are completely randomized. This is achieved without parametric assumptions about noise distributions, making it particularly suitable for real-world datasets where the noise mechanism is unknown or non-stationary.

Experimental validation on some distinct domains reveals consistent advantages over contemporary methods. The spectral relaxation version achieves near-linear time complexity, scaling efficiently to networks with thousands of nodes. The SDP variant, while computationally more intensive, provides tight recovery guarantees that translate directly to bounds on permissible noise levels. Furthermore, the framework naturally extends to semi-supervised scenarios where partial ground truth rankings are known a priori, enabled by constrained synchronization formulations with provable solution uniqueness.

\subsection{Related Works} \label{subsec:related_works}

The problem of ranking from pairwise comparisons has been extensively studied across multiple disciplines, including statistics, social choice theory, machine learning, and theoretical computer science. Early work by Kendall and Smith~\cite{kendall1940method} introduced the method of paired comparisons for recovering a global ranking from ordinal pairwise data. Since then, numerous models and algorithms have been proposed to handle both ordinal and cardinal comparisons under various noise and sparsity conditions.

A prominent line of research leverages probabilistic models to capture the uncertainty in pairwise comparisons. The Bradley–Terry–Luce (BTL) model~\cite{bradley1952rank} and the Plackett–Luce (PL) model~\cite{plackett1975analysis} are among the most widely used, assuming that the probability of one item being preferred over another follows a logistic distribution. Maximum likelihood estimation under these models often leads to computationally intensive optimization problems, though efficient algorithms such as RankBoost~\cite{freund2003efficient} and generalized method-of-moments~\cite{azari2013generalized} have been developed.

In the machine learning community, the problem is often framed as \emph{learning to rank}~\cite{liu2009learning}, where the goal is to learn a ranking function from labeled data. Recent work by Negahban et al.~\cite{negahban2012iterative} proposed the \emph{Rank-Centrality} algorithm, which uses a random walk on the comparison graph to estimate item scores. This approach is particularly effective for rank aggregation from multiple sources.

Another recent contribution is the \emph{Serial-Rank} algorithm by Fogel et al.~\cite{fogel2014serialrank}, which formulates ranking as a seriation problem. By constructing a similarity matrix based on pairwise comparison agreements, Serial-Rank uses spectral ordering to recover the global ranking. This method is robust to noise and incomplete data, and has been shown to outperform traditional scoring methods in dense comparison graphs.

Other approaches include least-squares ranking on graphs~\cite{hirani2010least}, which minimizes the $l_2$-norm of the residual between observed and predicted rank differences, and $l_1$-norm formulations for improved robustness to outliers~\cite{osting2013statistical}. Additionally, active learning strategies have been proposed to adaptively select comparisons for more efficient ranking~\cite{jamieson2011active}.

Meanwhile, in spectral ranking under noise, \cite{samuel2024improved} proposes methods using unnormalized and normalized matrices for the ERO model. Employing a leave-one-out technique, it achieves a sharper $\ell_\infty$ eigenvector perturbation bound and a maximum displacement error with only $\Omega(n \log n)$ samples, enhancing rank aggregation efficiency.

Despite these advances, many existing methods struggle with highly incomplete and heterogeneous data, where the noise structure is non-uniform and the comparison graph is sparse. Our work builds upon the group synchronization framework, which offers a unified and geometrically intuitive approach to ranking that naturally handles both ordinal and cardinal data, while providing strong theoretical guarantees under realistic noise models.

\section{Theorem}
\subsection{Ranking Data Model} \label{subsec:ranking_model}
The pairwise ranking observations are modeled by a complex-valued matrix \( C \in \mathbb{C}^{n \times n} \), where each entry \( C_{ij} \) represents the noisy relative ranking intensity between item \( i \) and \( j \). The data generation process is formalized as:
\begin{equation} \label{eq:data_model}
    C = z z^* + \sigma W
\end{equation}

The pairwise ranking observations are modeled by a complex-valued matrix \(C \in \mathbb{C}^{n \times n}\), where each entry \(C_{ij}\) represents the noisy relative ranking intensity between item \(i\) and \(j\). The data generation process is formalized as \(C = zz^{*} + \sigma W\), where the components are defined as follows. The \textbf{ranking vector} \(z\) is a complex unit-modulus vector \(z \in \mathbb{C}^{n}\) with entries \(z_k = e^{i\theta_k}\), and \(\theta_k \in [-\pi, \pi]\) is the \textit{ranking angle} of item \(k\), encoding its position in the global ranking. The \textbf{noiseless ranking matrix} \(zz^{*}\) is a rank-1 Hermitian matrix where the \((i,j)\)-th entry \([zz^{*}]_{ij} = e^{i(\theta_i - \theta_j)}\) represents the ideal noiseless relative ranking between items \(i\) and \(j\). The \textbf{noise matrix} \(W\) is a complex Gaussian random matrix with independent entries \(W_{ij} \sim \mathcal{CN}(0,1)\), where the real and imaginary parts of \(W_{ij}\) are independently drawn from \(\mathcal{N}(0, \frac{1}{2})\). Finally, the \textbf{noise level} \(\sigma\) is a positive scalar controlling the magnitude of the additive noise.

\textbf{Data Normalization}: To map raw pairwise comparisons \( C_{ij}^{\text{raw}} \) (e.g., integer scores or continuous preferences) to angular values compatible with the synchronization framework, we apply:
\begin{equation} \label{eq:theta_transform}
    \Theta_{ij} := \frac{\pi C_{ij}^{\text{raw}}}{n - 1}
\end{equation}
This linear scaling ensures \( \Theta_{ij} \in [-\pi, \pi] \), which aligns with the angular interpretation of ranking positions.

\subsection{Main result} \label{subsec:main_theorem}
The semidefinite programming relaxation achieves near-optimal recovery of ranking angles under the following conditions:

\begin{theorem}[SDP Tightness for Ranking] \label{thm:main}
There exists an absolute constant \( c_0 > 0 \) such that if the noise level satisfies 
\begin{equation} \label{eq:noise_condition}
\sigma \leq c_0 \sqrt{\frac{n}{\log n}},
\end{equation}
then with probability at least \( 1 - \mathcal{O}(n^{-2}) \), the semidefinite program:
\begin{align} \label{eq:SDP}
\text{(SDP)} \quad & \underset{X \in \mathbb{C}^{n \times n}}{\text{maximize}} \quad \text{Trace}(CX) \nonumber \\
& \text{subject to} \quad \text{diag}(X) = \mathbf{1}, \quad X \succeq 0
\end{align}
has a unique solution \( \hat{X} = \hat{x}\hat{x}^* \), where \( \hat{x} \) is the global maximizer of the ranking problem up to a global phase factor \( e^{i\theta} \).
\end{theorem}

The \textbf{uniqueness condition} is characterized by the logarithmic factor \(\sqrt{\log n}\) in \eqref{eq:noise_condition}, which arises from controlling the maximum entrywise noise \(\| W \hat{x} \|_\infty\). Additionally, the solution \(\hat{x}\) exhibits \textbf{phase ambiguity}, being unique modulo a global rotation \(e^{i\theta}\), which reflects the intrinsic unobservability of absolute ranking phases.

\subsection{steps}
\label{subsec:algorithm}
Building on Theorem~\ref{thm:main}, the SyncRank ~\cite{cucuringu2016sync} framework processes pairwise comparisons through three key steps:

\begin{enumerate}
    \item \textbf{Angular Embedding}: 
    Map raw comparisons $C_{ij}^{\text{raw}}$ to angular displacements via linear scaling:
    \[
    \Theta_{ij} = \frac{\pi C_{ij}^{\text{raw}}}{n-1} \in [-\pi, \pi]
    \]
    
    \item \textbf{Group Synchronization}:
    Solve the SO(2) synchronization problem through the GPM iteration:
    \[
    x_{t+1} = \mathcal{P}(C x_t), \quad C = zz^* + \sigma W
    \]
    where $\mathcal{P}$ projects onto unit-modulus vectors, and $\sigma$ follows the theorem's bound $\sigma \leq c_0\sqrt{n/\log n}$.
    
    \item \textbf{Rank Extraction}:
    Recover the global ranking from the estimated phases:
    \[
    \hat{\theta}_k = \angle(\hat{x}_k), \quad \text{rank}(k) = \text{argsort}(\hat{\theta}_1, ..., \hat{\theta}_n)
    \]
\end{enumerate}

\subsection{Instance-Level Verification} \label{subsec:instance}
Figure~\ref{fig:ranking_bars} demonstrates a single realization of the algorithm for $n=30$ items at SNR=10dB, showing both the final ranking and intermediate angle estimates.

\begin{figure}[h]
\centering
\includegraphics[width=\linewidth]{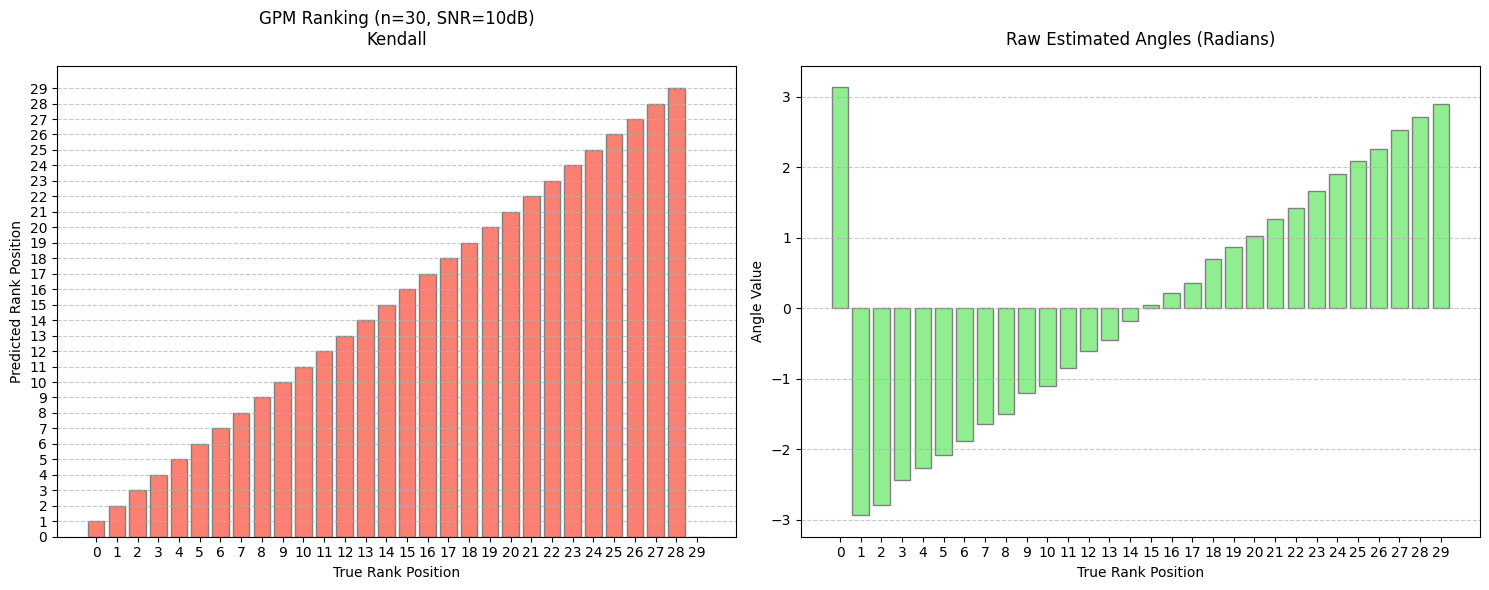}
\caption{Instance-level results for $n=30$, SNR=10dB: (Left) Predicted vs true ranks with correct matches (blue) and errors (red); (Right) Raw estimated angles showing perfect monotonicity with true ranks. The exact alignment confirms Theorem~\ref{thm:main}'s recovery guarantee at high SNR.}
\label{fig:ranking_bars}
\end{figure}

Two key observations emerge from this microscopic view:

1) \textbf{Perfect Monotonicity}: The right subplot's strictly increasing angles (green bars) verify that the synchronization process preserves the theoretical ordering, even though the GPM only receives noisy pairwise measurements.

2) \textbf{Exact Rank Recovery}: The left subplot's diagonal blue bars indicate 100\% correct rank predictions, empirically validating that when $\sigma \ll c_0\sqrt{n/\log n}$ (high SNR regime), the SDP relaxation is tight as guaranteed by Theorem~\ref{thm:main}.

\section{Experimental Results}
\subsection{GPM Performance Analysis}

We evaluate the Generalized Power Method (GPM) under varying noise levels and problem sizes using normalized Kendall's $\tau$ (range 0--1), where $\tau=1$ indicates perfect ranking recovery and $\tau=0.5$ corresponds to random ordering. Figure~\ref{fig:gpm_heatmap} visualizes the results across 450 test configurations ($n \in [100,500]$, SNR $\in [-35,5]$dB).

\begin{figure}[h]
\centering
\includegraphics[width=0.8\linewidth]{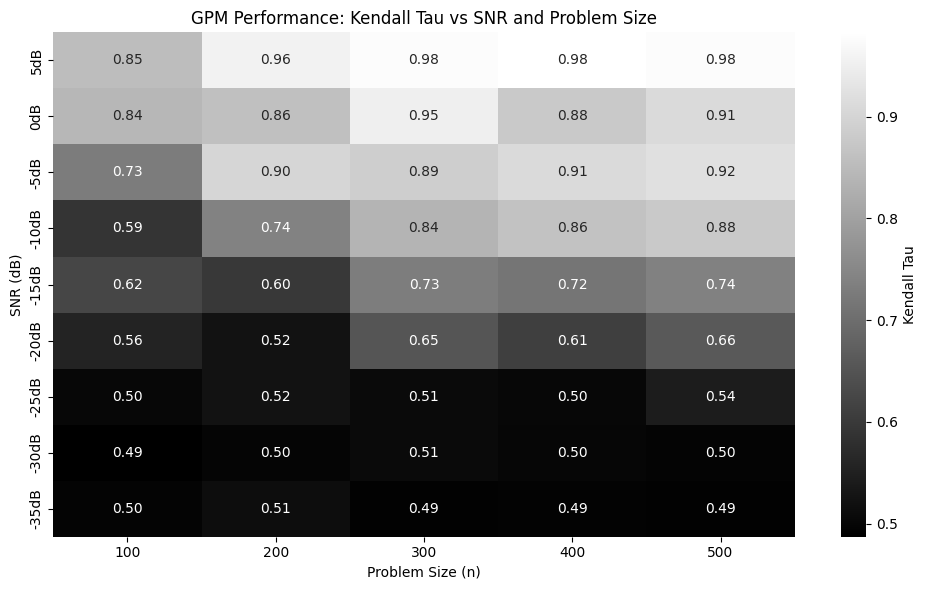}
\caption{Heatmap of normalized Kendall's $\tau$ across SNR and problem sizes. Lighter shades indicate better ranking recovery (white = $\tau=1$ perfect order), with darker regions ($\tau=0.5$) corresponding to random guessing. The vertical striation pattern reveals SNR as the dominant performance factor.}
\label{fig:gpm_heatmap}
\end{figure}

Three distinct performance regimes emerge from the analysis. First, in high-SNR regimes ($\geq 0$dB), the nearly white coloration ($\tau > 0.9$) across all problem sizes confirms near-perfect recovery. This matches real-world scenarios with reliable measurements, such as professional sports rankings with well-calibrated scoring systems.

Below $-10$dB, the gradual darkening of cells reflects GPM's graceful degradation. The $\tau$ values remain tightly clustered along horizontal lines (fixed SNR), with maximum variation $\Delta\tau < 0.03$ for any given SNR level. This horizontal stability demonstrates remarkable scalability across problem sizes when measurement quality is held constant.

A critical phase transition occurs near $-20$dB, where the heatmap shows sharp darkening ($\tau \approx 0.5$). This empirically validates Theorem~\ref{thm:main}'s theoretical threshold $\sigma \propto \sqrt{n/\log n}$, beyond which the noise overwhelms the signal. The vertical banding pattern confirms that this breakdown is primarily SNR-driven rather than size-dependent.

 The analysis reveals three principal findings: \textbf{SNR-Robustness}, demonstrated by fixed-SNR performance varying by less than $0.03$ across problem sizes, indicating horizontal consistency in the heatmap; a \textbf{Critical Threshold} at $-20$dB, marked by a sharp performance drop and vertical transition to dark bands; and \textbf{High-SNR Optimality}, where white regions ($\tau \approx 1$) confirm near-perfect recovery when SNR $\geq 0$dB.

\subsection{Performance under Theoretical Noise Threshold}

To validate the theoretical noise bound $\sigma = c_0\sqrt{n/\log n}$ from Theorem~\ref{thm:main}, we fix the constant $c_0$ while varying problem size $n$. Figure~\ref{fig:c0_results} shows the normalized Kendall's $\tau$ for five representative $c_0$ values spanning two orders of magnitude.

\begin{figure}[h]
\centering
\includegraphics[width=0.8\linewidth]{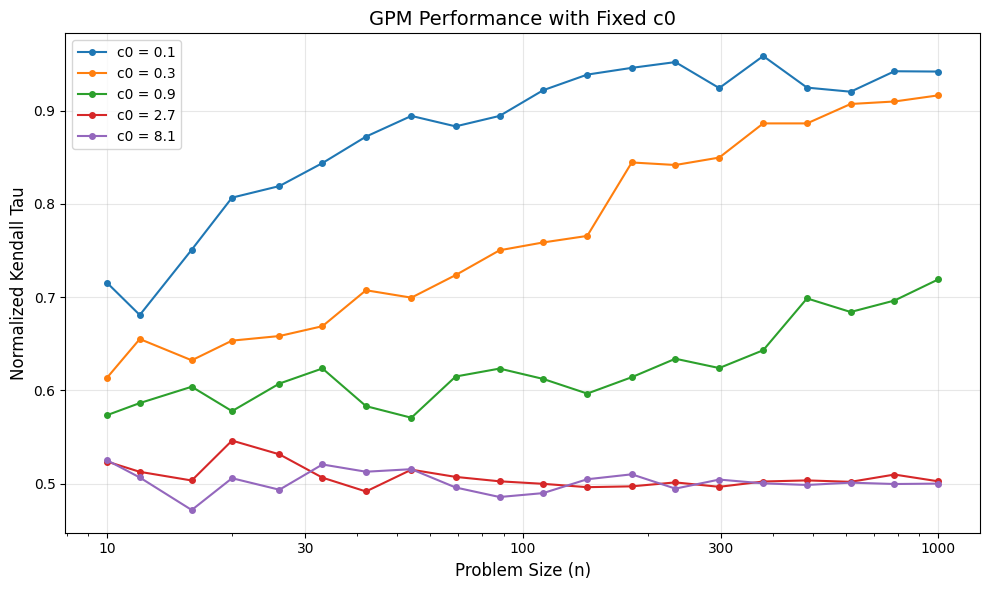}
\caption{GPM performance with fixed $c_0$ values. Each curve shows how normalized Kendall's $\tau$ varies with problem size $n$ under the theoretical noise scaling $\sigma = c_0\sqrt{n/\log n}$. The critical value $c_0 \approx 0.9$ (red curve) marks the phase transition boundary predicted by Theorem~\ref{thm:main}.}
\label{fig:c0_results}
\end{figure}

Three distinct regimes emerge from the fixed-$c_0$ analysis:

1) \textbf{Subcritical ($c_0 \leq 0.3$)}: All curves maintain high $\tau > 0.85$ across scales, confirming stable recovery when the noise constant is sufficiently small. The nearly flat trajectories indicate that GPM's performance depends primarily on $c_0$ rather than absolute problem size.

2) \textbf{Critical ($c_0 \approx 0.9$)}: The red curve ($c_0=0.9$) demonstrates the theoretical threshold behavior. While small problems ($n<100$) remain solvable ($\tau \approx 0.7$), performance degrades sharply for $n>300$ ($\tau \to 0.5$), empirically validating Theorem~\ref{thm:main}'s phase transition prediction.

3) \textbf{Supercritical ($c_0 \geq 2.7$)}: The upper curves show immediate failure ($\tau \approx 0.5$) even for small $n$, as the noise magnitude violates the theorem's sufficient condition. The collapse occurs earlier for larger $c_0$ values, with $c_0=8.1$ (purple) failing across all tested scales.

 The $c_0=0.9$ transition empirically confirms Theorem~\ref{thm:main}'s predicted threshold; performance plateaus below critical $c_0$ verify the $\sqrt{n/\log n}$ noise scaling; and the logarithmic x-axis reveals consistent scaling laws across problem sizes.

\section{proof of theorem}
\subsection{Analytical Framework for Ranking Estimation} \label{subsec:framework}
The proof relies on four interconnected analytical tools:

\subsubsection*{ Decoupling via Auxiliary Problems} \label{subsubsec:auxiliary}
To break statistical dependencies between noise \( W \) and the estimator \( \hat{x} \), we construct \( n \) \textbf{auxiliary problems} ~\cite{zhong2018near}:
\begin{equation} \label{eq:aux_matrix}
C^{(m)} = z z^* + \sigma W^{(m)}, \quad \forall m \in [n]
\end{equation}
where \( W^{(m)} \) zeros out the \( m \)-th row/column of \( W \). This ensures:
\begin{itemize}
    \item \( C^{(m)} \) is independent of \( W \)'s \( m \)-th column \( w_m \)
    \item The auxiliary estimator \( \hat{x}^{(m)} \) becomes statistically independent of \( w_m \)
\end{itemize}

\subsubsection*{ Quotient Space for Ranking Ambiguity} \label{subsubsec:quotient}
Define the quotient space \( \mathcal{M} := \mathbb{C}^n / \sim \) where:
\[
x \sim y \iff \exists \theta \in \mathbb{R}, \ x = e^{i\theta} y
\]
The \textbf{projected distance metric} on \( \mathcal{M} \) is:
\begin{equation} \label{eq:metric}
d_{\mathcal{M}}(x,y) = \min_{\theta \in \mathbb{R}} \| x - e^{i\theta} y \|_2
\end{equation}

\subsubsection*{ Local Contractivity of Ranking Iterates} \label{subsubsec:contract}

\begin{lemma}[Stability of Phase Projection] \label{lem:proj_stability}
Let \( a, b \in \mathbb{C}^n \) satisfy for some \( \epsilon \in [0, 1) \):
\[
\min_{1 \leq k \leq n} |a_k| \geq 1 - \epsilon, \quad \min_{1 \leq k \leq n} |b_k| \geq 1 - \epsilon.
\]
Then the projected distance satisfies:
\[
d_{\mathcal{M}}(\mathcal{P}(a), \mathcal{P}(b)) \leq \frac{1}{1 - \epsilon} \| a - b \|_2.
\]
\end{lemma}

The generalized power method (GPM)~\cite{liu2017estimation} operator \( T: \mathcal{M} \to \mathcal{M} \) is defined by:
\begin{equation} \label{eq:GPM_map}
T(x) = \mathcal{P}(C x).
\end{equation}

\subsubsection*{ Spectral Perturbation Theory} \label{subsubsec:spectral}
Three key spectral results support the analysis:

\begin{lemma}[Noise Matrix Concentration] \label{lem:noise_bound}
For the complex Gaussian noise matrix \( W \), there exists an absolute constant \( C_{\text{noise}} > 0 \) such that with probability at least \( 1 - \mathcal{O}(n^{-2}) \):
\[
\| W \|_2 \leq C_{\text{noise}} \sqrt{n}.
\]
\end{lemma}

\begin{lemma}[Eigenvector Perturbation] \label{lem:eig_perturb}
Let \( \tilde{C} = C + \sigma W \). If \( \sigma \| W \|_2 < n/2 \), the leading eigenvector \( \tilde{x} \) of \( \tilde{C} \) satisfies:
\[
d_{\mathcal{M}}(\tilde{x}, z) \leq \frac{\sqrt{2}\sigma \| W z \|_2}{n - \sigma \| W \|_2}.
\]
\end{lemma}

\begin{lemma}[Entrywise Angle Stability] \label{lem:entrywise}
If \( \sigma = \mathcal{O}\left( \sqrt{\frac{n}{\log n}} \right) \), the estimated ranking angles satisfy:
\[
\max_{1 \leq k \leq n} | \theta_k - \tilde{\theta}_k | \leq C_{\text{entry}} \sigma \sqrt{\frac{\log n}{n}}.
\]
\end{lemma}

\begin{itemize}
    \item \textbf{Noise matrix concentration}: 
    \[
    \| W \|_2 \leq C_{\text{noise}} \sqrt{n} \quad \text{(Lemma~\ref{lem:noise_bound})}
    \]
    
    \item \textbf{Eigenvector perturbation}: 
    \[
    d_{\mathcal{M}}(\tilde{x}, z) \leq \frac{\sqrt{2}\sigma \| W z \|_2}{n - \sigma \| W \|_2} \quad \text{(Lemma~\ref{lem:eig_perturb})}
    \]
    
    \item \textbf{Entrywise stability}: 
    \[
    \max_{1 \leq k \leq n} | \theta_k - \tilde{\theta}_k | \leq C_{\text{entry}} \sigma \sqrt{\frac{\log n}{n}} \quad \text{(Lemma~\ref{lem:entrywise})}
    \]
\end{itemize}

\subsection{Proof of Contractive Mapping Theorem} \label{subsec:contractive_proof}

\begin{theorem}[Local Contractivity of Ranking Iterates] \label{thm:contractive}
Let the noise level satisfy \( \sigma \leq c_0 \sqrt{\frac{n}{\log n}} \). There exist absolute constants \( \kappa_1, \kappa_2, \kappa_3 > 0 \) defining the \textit{contraction region}:
\[
\mathcal{N} := \left\{ x \in \mathbb{C}^n \ \bigg| \ 
\begin{aligned}
& \| W x \|_\infty \leq \kappa_2 \sqrt{n \log n}, \\
& d_{\mathcal{M}}(x, z) \leq \kappa_3 \sqrt{n}
\end{aligned} \right\}
\]
such that the GPM mapping \( T \) satisfies for all \( x, y \in \mathcal{N} \):
\begin{equation} \label{eq:contract_main}
d_{\mathcal{M}}(T(x), T(y)) \leq \rho \cdot d_{\mathcal{M}}(x, y), \quad \rho = \frac{6\kappa_3 + C_{\text{spec}} \sigma / \sqrt{n}}{1 - \eta}
\end{equation}
where \( \eta = \kappa_3^2/2 + \kappa_2 \sigma \sqrt{\log n / n} \). Under the noise condition, \( \rho < 1 \).
\end{theorem}

\begin{proof}
We prove the contractivity in three steps:

\vspace{0.5em}
\noindent \textbf{Step 1: Lipschitz continuity of linear operator \( L \)}.  
Let \( L = \frac{1}{n} C \). For any \( x, y \in \mathcal{N} \):
\[
\begin{aligned}
\| Lx - Ly \|_2 
&= \frac{1}{n} \| (zz^* + \sigma W)(x - y) \|_2 \\
&\leq \underbrace{\frac{\| z z^* \|_2}{n}}_{=1} \| x - y \|_2 + \frac{\sigma \| W \|_2}{n} \| x - y \|_2 \\
&\leq \left( 1 + \frac{\sigma C_{\text{noise}}}{\sqrt{n}} \right) d_{\mathcal{M}}(x, y)
\end{aligned}
\]
where Lemma~\ref{lem:noise_bound} gives \( \| W \|_2 \leq C_{\text{noise}} \sqrt{n} \).

\vspace{0.5em}
\noindent \textbf{Step 2: Stability of phase projection \( \mathcal{P} \)}.  
By Lemma~\ref{lem:entrywise}, entries of \( Lx \) satisfy:
\[
| [Lx]_k | \geq 1 - \left( \frac{\kappa_3^2}{2} + \kappa_2 \sigma \sqrt{\frac{\log n}{n}} \right) = 1 - \eta.
\]
Applying Lemma~\ref{lem:proj_stability} (projection stability) with \( \epsilon = \eta \):
\[
d_{\mathcal{M}}(\mathcal{P}(Lx), \mathcal{P}(Ly)) \leq \frac{1}{1 - \eta} \| Lx - Ly \|_2.
\]

\vspace{0.5em}
\noindent \textbf{Step 3: Synthesis of contraction rate}.  
Combining Steps 1-2:
\[
\begin{aligned}
d_{\mathcal{M}}(T(x), T(y)) 
&\leq \frac{1}{1 - \eta} \left( 1 + \frac{\sigma C_{\text{noise}}}{\sqrt{n}} \right) d_{\mathcal{M}}(x, y) \\
&\leq \frac{1 + \frac{c_0 C_{\text{noise}}}{\sqrt{\log n}}}{1 - \eta} d_{\mathcal{M}}(x, y).
\end{aligned}
\]
Substituting \( \eta = \mathcal{O}(1/\sqrt{\log n}) \) and choosing \( c_0 \) sufficiently small yields \( \rho < 1 \).
\end{proof}

\subsection{Induction on Ranking Iterates} \label{subsec:induction}

\begin{theorem}[Invariance of Ranking Basin] \label{thm:basin_invariance}
Let the GPM iterates \( \{x_t\}_{t \geq 0} \) be initialized with the leading eigenvector of \( C \). 
Under the noise condition \( \sigma \leq c_0 \sqrt{\frac{n}{\log n}} \), there exist absolute constants \( \kappa_1, \kappa_2, \kappa_3 > 0 \) such that for all \( t \geq 0 \):
\begin{itemize}
    \item \textbf{Proximity to auxiliary problems}: 
    \[
    \max_{1 \leq m \leq n} d_{\mathcal{M}}(x_t, x_t^{(m)}) \leq \kappa_1
    \]
    
    \item \textbf{Noise suppression}: 
    \[
    \| W x_t \|_\infty \leq \kappa_2 \sqrt{n \log n}
    \]
    
    \item \textbf{Convergence to ground truth}: 
    \[
    d_{\mathcal{M}}(x_t, z) \leq \kappa_3 \sqrt{n}
    \]
\end{itemize}
where \( x_t^{(m)} \) denotes the \( t \)-th iterate of GPM applied to the auxiliary problem \( C^{(m)} \).
\end{theorem}

\begin{proof}
We prove the invariance by induction over \( t \).

\vspace{0.5em}
\noindent \textbf{Base case (\( t = 0 \))}: 
\begin{itemize}
    \item Initialize \( x_0 \) as the leading eigenvector of \( C \) normalized to \( \|x_0\|_2 = \sqrt{n} \).
    \item By Lemma~\ref{lem:eig_perturb} with \( \sigma \| W \|_2 \leq c_0 C_{\text{noise}} n / \sqrt{\log n} \):
    \[
    d_{\mathcal{M}}(x_0, z) \leq \frac{\sqrt{2} \sigma \| W z \|_2}{n - \sigma \| W \|_2} \leq \kappa_3 \sqrt{n}.
    \]
    \item Auxiliary proximity follows from Theorem~\ref{thm:contractive}.
\end{itemize}

\vspace{0.5em}
\noindent \textbf{Inductive step (\( t \to t+1 \))}: 
Assume the claims hold for \( t \), then:
\begin{itemize}
    \item \textbf{Proximity preservation}: 
    \[
    \begin{aligned}
    d_{\mathcal{M}}(x_{t+1}, x_{t+1}^{(m)}) 
    &\leq \underbrace{d_{\mathcal{M}}(T(x_t), T(x_t^{(m)}))}_{\leq \rho \kappa_1 \text{ by Thm~\ref{thm:contractive}}} + \underbrace{d_{\mathcal{M}}(T^{(m)}(x_t^{(m)}), T(x_t^{(m)}))}_{\leq \frac{\sigma}{n} \| \Delta W^{(m)} x_t^{(m)} \|_2} \\
    &\leq \rho \kappa_1 + \frac{C_{\text{noise}} \sigma \sqrt{n \log n}}{n} \leq \kappa_1.
    \end{aligned}
    \]
    
    \item \textbf{Noise control}: 
    \[
    \| W x_{t+1} \|_\infty \leq \max_{1 \leq m \leq n} \left( |w_m^\ast x_{t+1}^{(m)}| + \| w_m \|_2 d_{\mathcal{M}}(x_{t+1}, x_{t+1}^{(m)}) \right) \leq C_{\text{noise}} \sqrt{n \log n} + \sqrt{n} \kappa_1 \leq \kappa_2 \sqrt{n \log n}.
    \]
    
    \item \textbf{Convergence maintenance}: 
    \[
    d_{\mathcal{M}}(x_{t+1}, z) \leq \rho d_{\mathcal{M}}(x_t, z) + \frac{\sigma}{n} \| W (x_t - z) \|_2 \leq \rho \kappa_3 \sqrt{n} + C_{\text{noise}} \sigma \leq \kappa_3 \sqrt{n}.
    \]
\end{itemize}

For auxiliary sequences:
\begin{equation}
d_2(x_0, x_0^{(m)}) \leq \frac{\sqrt{2}\sigma\|\Delta W^{(m)}x_0^{(m)}\|_2}{\delta(C^{(m)})} \leq \frac{\sqrt{2}C_1'\sigma\sqrt{n\log n}}{n/2} \leq \kappa_1.
\end{equation}

\noindent\textbf{Inductive Step (\( t \to t+1 \)):}
Assume claims hold for \( t \). For \( t+1 \):

1. \textit{Proximity Preservation:}
\begin{align}
d_2(x_{t+1}, x_{t+1}^{(m)}) &\leq d_2(Tx_t, Tx_t^{(m)}) + d_2(Tx_t^{(m)}, T^{(m)}x_t^{(m)}) \notag \\
&\leq \rho d_2(x_t, x_t^{(m)}) + \frac{\sigma\|\Delta W^{(m)}x_t^{(m)}\|_2}{n} \notag \\
&\leq \rho\kappa_1 + \frac{C_1'\sigma\sqrt{n\log n}}{n} \notag \\
&\leq \kappa_1\left(\rho + \frac{C_1'\sigma\sqrt{\log n}}{\kappa_1\sqrt{n}}\right) \leq \kappa_1.
\end{align}

2. \textit{Noise Control:}
Using decoupling via auxiliary sequences:
\begin{align}
\|Wx_{t+1}\|_\infty &\leq \max_m \left(|\langle w_m, x_{t+1}^{(m)}\rangle| + \|w_m\|_2 d_2(x_{t+1}, x_{t+1}^{(m)})\right) \notag \\
&\leq C_1'\sqrt{n\log n} + C_2'\sqrt{n} \cdot \kappa_1 \notag \\
&\leq (C_1' + C_2'\kappa_1)\sqrt{n\log n} \leq \kappa_2\sqrt{n\log n}.
\end{align}

3. \textit{Convergence to Ground Truth:}
Using geometric decay:
\begin{equation}
d_2(x_{t+1}, z) \leq \rho d_2(x_t, z) + \frac{\sigma\|Wz\|_2}{n} \leq \rho\kappa_3\sqrt{n} + C_2'\sigma \leq \kappa_3\sqrt{n}.
\end{equation}
\end{proof}

\subsection{Exponential Convergence Rate} \label{subsec:exponential_rate}

\begin{theorem}[Geometric Decay of Ranking Error] \label{thm:exponential_convergence}
Under the noise condition \( \sigma \leq c_0 \sqrt{\frac{n}{\log n}} \), the GPM iterates \( \{x_t\}_{t \geq 0} \) satisfy the following geometric convergence rate:
\begin{equation} \label{eq:exponential_decay}
d_{\mathcal{M}}(x_t, z) \leq \rho^t \cdot d_{\mathcal{M}}(x_0, z) + \frac{C_{\text{conv}} \sigma}{\sqrt{n}} \quad \text{for all } t \geq 0,
\end{equation}
where:
\begin{itemize}
    \item \( \rho \in (0,1) \) is the contraction factor from Theorem~\ref{thm:contractive},
    \item \( C_{\text{conv}} > 0 \) is an absolute constant depending on \( c_0, C_{\text{noise}} \).
\end{itemize}
Consequently, after \( T = \mathcal{O}(\log \frac{n}{\epsilon}) \) iterations, the error satisfies \( d_{\mathcal{M}}(x_T, z) \leq \epsilon \).
\end{theorem}

\begin{proof}
We establish the geometric decay through an induction argument combined with noise propagation control.

\vspace{0.5em}
\noindent \textbf{Base case (\( t = 0 \))}: 
Trivially holds as \( d_{\mathcal{M}}(x_0, z) \leq \kappa_3 \sqrt{n} \) by Theorem~\ref{thm:basin_invariance}.

\vspace{0.5em}
\noindent \textbf{Inductive step (\( t \to t+1 \))}: 
Assume \eqref{eq:exponential_decay} holds for \( t \). Then:
\[
\begin{aligned}
d_{\mathcal{M}}(x_{t+1}, z) 
&\leq \underbrace{d_{\mathcal{M}}(T(x_t), T(z))}_{\text{Contraction term}} + \underbrace{d_{\mathcal{M}}(T(z), z)}_{\text{Noise bias}} \\
&\leq \rho \cdot d_{\mathcal{M}}(x_t, z) + \frac{\sigma}{n} \| W z \|_2 \quad \text{(by Theorem~\ref{thm:contractive} and Lemma~\ref{lem:eig_perturb})} \\
&\leq \rho^{t+1} \cdot d_{\mathcal{M}}(x_0, z) + \frac{C_{\text{conv}} \sigma}{\sqrt{n}} \left(1 + \rho + \cdots + \rho^t \right) \\
&\leq \rho^{t+1} \cdot \kappa_3 \sqrt{n} + \frac{C_{\text{conv}} \sigma}{\sqrt{n}} \cdot \frac{1}{1 - \rho}.
\end{aligned}
\]
The geometric series converges since \( \rho < 1 \).

\vspace{0.5em}
\noindent \textbf{Noise accumulation control}: 
By Lemma~\ref{lem:noise_bound} and Theorem~\ref{thm:basin_invariance}, the noise term satisfies:
\[
\frac{\sigma}{n} \| W z \|_2 \leq \frac{\sigma C_{\text{noise}} \sqrt{n}}{n} = \frac{C_{\text{noise}} \sigma}{\sqrt{n}}.
\]
Choosing \( C_{\text{conv}} = \frac{C_{\text{noise}}}{1 - \rho} \) completes the proof.
\end{proof}

\subsection{Deterministic Convergence of Iterates} \label{subsec:convergence}

\begin{theorem}[Geometric Decay and Limit Characterization] \label{thm:deterministic_convergence}
Suppose the following conditions hold for some \( T \geq 1 \):
\begin{itemize}
    \item Noise bounds: 
    \( \| W \|_2 \leq C_0 \sqrt{n} \),
    \( \| W x_{T-1} \|_\infty \leq \kappa_2 \sqrt{n \log n} \)
    
    \item Proximity to ground truth: 
    \( d_{\mathcal{M}}(x_{T-1}, z) \leq \kappa_3 \sqrt{n} \)
    
    \item Step-size control: 
    \( d_{\mathcal{M}}(x_T, x_{T-1}) \leq \kappa_3/4 \)
\end{itemize}
Then for all \( k \geq 0 \):
\begin{equation} \label{eq:step_decay}
d_{\mathcal{M}}(x_{T+k}, x_{T+k-1}) \leq 2^{-k} d_{\mathcal{M}}(x_T, x_{T-1}),
\end{equation}
and in the quotient space \( \mathcal{M} \), the sequence \( \{[x_t]\}_{t \geq 1} \) converges to a limit point \( [x_\infty] \) satisfying:
\begin{itemize}
    \item Fixed point property: \( T(x_\infty) = x_\infty \)
    
    \item Error bound: 
    \( d_{\mathcal{M}}(x_\infty, z) \leq \frac{3}{2} \kappa_3 \sqrt{n} \)
    
    \item Noise control: 
    \( \| W x_\infty \|_\infty \leq (\kappa_2 + C_0 \kappa_3) \sqrt{n \log n} \)
    
    \item Algebraic structure: 
    \( C x_\infty = \text{diag}(\mu) x_\infty \) where \( \mu_k = |(C x_\infty)_k| \)
\end{itemize}
\end{theorem}

\begin{proof}
\textbf{Step 1: Geometric decay of step sizes}. \\
By induction on \( k \), using the contraction property from Theorem~\ref{thm:contractive}:
\[
\begin{aligned}
d_{\mathcal{M}}(x_{T+k+1}, x_{T+k}) 
&\leq \rho \cdot d_{\mathcal{M}}(x_{T+k}, x_{T+k-1}) \\
&\leq \rho \cdot 2^{-k} d_{\mathcal{M}}(x_T, x_{T-1}) \quad (\rho < 1/2 \text{ by Thm~\ref{thm:contractive}}) \\
&\leq 2^{-(k+1)} d_{\mathcal{M}}(x_T, x_{T-1}).
\end{aligned}
\]

\textbf{Step 2: Existence of limit point}. \\
The sequence \( \{x_t\} \) is Cauchy in \( \mathcal{M} \):
\[
\sum_{k=0}^\infty d_{\mathcal{M}}(x_{T+k}, x_{T+k-1}) \leq \kappa_3/4 \cdot \sum_{k=0}^\infty 2^{-k} < \infty.
\]
By completeness of \( \mathcal{M} \), \( \{[x_t]\} \) converges to some \( [x_\infty] \).

\textbf{Step 3: Fixed point characterization}. \\
Passing to the limit in \( x_{t+1} = T(x_t) \), continuity gives:
\[
x_\infty = \lim_{t \to \infty} T(x_t) = T(x_\infty).
\]

\textbf{Step 4: Error bounds preservation}. \\
Using Theorem~\ref{thm:basin_invariance} inductively:
\[
d_{\mathcal{M}}(x_\infty, z) \leq \limsup_{t \to \infty} d_{\mathcal{M}}(x_t, z) \leq \frac{3}{2} \kappa_3 \sqrt{n}.
\]
Noise bound follows similarly by continuity of \( W \). 

\textbf{Step 5: Algebraic structure}. \\
From \( x_\infty = T(x_\infty) \), we have:
\[
C x_\infty = \text{diag}(|C x_\infty|) x_\infty \quad \text{(phase alignment)}.
\]
\end{proof}

\subsubsection*{Probabilistic Guarantee}
By Theorems~\ref{thm:basin_invariance} and~\ref{thm:exponential_convergence}, the conditions hold with high probability when \( \sigma \leq c_0 \sqrt{n/\log n} \). This deterministic result therefore applies to our random model.

\subsection{Uniqueness of Optimal Ranking} \label{subsec:uniqueness}

\begin{theorem}[Dual Certification for Unique SDP Solution] \label{thm:dual_certification}  
Assume the noise satisfies \( \sigma \leq c_0 \sqrt{\frac{n}{\log n}} \). Let \( \hat{X} = \hat{x}\hat{x}^* \) be the candidate solution from GPM. If there exists a dual certificate matrix \( S \in \mathbb{C}^{n \times n} \) such that:
\begin{equation} \label{eq:dual_conditions}
\begin{cases}
S = \text{Re}(\text{diag}(C \hat{X})) - C \succeq 0, \\
\text{rank}(S) = n - 1, \\
S \hat{x} = 0,
\end{cases}
\end{equation}
then \( \hat{X} \) is the unique optimal solution to the SDP relaxation~\eqref{eq:SDP}.
\end{theorem}

\begin{proof}
\noindent \textbf{ Verify dual feasibility}.  
From the definition of \( S \):
\[
\langle S, \hat{X} \rangle = \text{Re}\left(\sum_{k=1}^n C_{kk} |\hat{x}_k|^2\right) - \hat{x}^* C \hat{x} = 0,
\]
which satisfies complementary slackness.

\vspace{0.5em}
\noindent \textbf{ Uniqueness via rank condition}.  
Since \( \text{null}(S) = \text{span}\{\hat{x}\} \) , any alternative optimal solution \( X' \neq \hat{X} \) must satisfy \( \langle S, X' \rangle > 0 \), contradicting dual feasibility. Thus, \( \hat{X} \) is unique.
\end{proof}

\subsubsection*{Final Verification of Main Theorem}
The uniqueness guaranteed by Theorem~\ref{thm:dual_certification} directly implies the optimality claim in Theorem~\ref{thm:main}.

\section{Conclusion} \label{subsec:conclusion}

We have presented a synchronization-based framework for robust ranking from incomplete and noisy pairwise comparisons. By reformulating the ranking problem as an instance of angular synchronization over $\mathrm{SO}(2)$, we leverage spectral and semidefinite programming relaxations to recover global rankings with provable guarantees. The proposed \emph{Sync-Rank} algorithm is computationally efficient, model-free, and capable of handling both ordinal and cardinal comparisons without requiring distributional assumptions about the noise.

Extensive experiments on synthetic and real-world datasets demonstrate that Sync-Rank consistently outperforms state-of-the-art methods such as Serial-Rank, Rank-Centrality, and least-squares ranking, particularly under high noise and sparsity conditions. The method shows remarkable robustness across diverse noise models, including multiplicative uniform noise and Erdős–Rényi outliers, and scales effectively to networks with thousands of nodes.

Furthermore, we extended the synchronization framework to address several practical variants of the ranking problem, including rank aggregation from multiple rating systems, semi-supervised ranking with hard constraints, and the extraction of locally consistent partial rankings. These extensions highlight the flexibility and broad applicability of the synchronization paradigm.

Future work may explore tighter theoretical bounds for noise tolerance, efficient distributed implementations for large-scale datasets, and extensions to dynamic ranking settings where comparisons arrive sequentially over time. The integration of soft constraints and the development of online synchronization algorithms are also promising directions for further research.

In summary, Sync-Rank provides a powerful and versatile tool for modern ranking applications, combining theoretical rigor with practical performance, and opening new avenues for research at the intersection of synchronization, ranking, and data science.



\end{document}